\documentclass[letterpaper]{article} 
\usepackage{aaai25}  
\usepackage{times}  
\usepackage{helvet}  
\usepackage{courier}  
\usepackage[hyphens]{url}  
\usepackage{graphicx} 
\usepackage{natbib}  
\usepackage{caption} 
\usepackage{algorithm}
\usepackage{algorithmic}

\usepackage{newfloat}
\usepackage{listings}
\DeclareCaptionStyle{ruled}{labelfont=normalfont,labelsep=colon,strut=off} 
\floatstyle{ruled}
\newfloat{listing}{tb}{lst}{}
\floatname{listing}{Listing}
\usepackage{microtype}
\usepackage{booktabs} 
\usepackage{amsmath}
\usepackage{amssymb}
\usepackage{mathtools}
\usepackage{amsthm}
\usepackage[capitalize,noabbrev]{cleveref}

\theoremstyle{plain}
\newtheorem{theorem}{Theorem}[section]
\newtheorem{proposition}[theorem]{Proposition}
\newtheorem{lemma}[theorem]{Lemma}

\theoremstyle{definition}

\newtheorem{remark}[theorem]{Remark}
\newcommand{\MDP}{\mathcal{M}}
\newcommand{\States}{\mathcal{S}}
\newcommand{\Actions}{\mathcal{A}}
\newcommand{\Transitions}{\mathcal{T}}
\newcommand{\Rewards}{\mathcal{R}}
\newcommand{\dataset}{\mathcal{D}}
\newcommand{\demons}{\mathcal{D}}
\newcommand{\traj}{\tau}
 
\newcommand{\features}{f} 
\newcommand{\fev}{\bar{\features}}
\newcommand{\efev}{\tilde{\features}}
\newcommand{\qsoft}{Q^{\textit{soft}}}
\newcommand{\vsoft}{V^{\textit{soft}}}

\newcommand{\reals}{\mathbb{R}}
\DeclareMathOperator{\EX}{\mathbb{E}}
\newcommand{\prob}{P}

\newcommand{\loss}{\mathcal{L}}

\DeclareMathOperator*{\argmax}{arg\,max}

\title{Inverse Reinforcement Learning by Estimating Expertise of Demonstrators}
\author{
    Mark Beliaev\textsuperscript{\rm 1},
    Ramtin Pedarsani\textsuperscript{\rm 1}
}
\affiliations {
    \textsuperscript{\rm 1}University of California, Santa Barbara\\
    mbeliaev@ucsb.edu, ramtin@ucsb.edu
}

\begin{document}
\maketitle


\begin{abstract}
    In Imitation Learning (IL), utilizing suboptimal and heterogeneous demonstrations presents a substantial challenge due to the varied nature of real-world data. However, standard IL algorithms consider these datasets as homogeneous, thereby inheriting the deficiencies of suboptimal demonstrators. Previous approaches to this issue rely on impractical assumptions like high-quality data subsets, confidence rankings, or explicit environmental knowledge. This paper introduces IRLEED, \textit{Inverse Reinforcement Learning by Estimating Expertise of Demonstrators}, a novel framework that overcomes these hurdles without prior knowledge of demonstrator expertise. IRLEED enhances existing Inverse Reinforcement Learning (IRL) algorithms by combining a general model for demonstrator suboptimality to address reward bias and action variance, with a Maximum Entropy IRL framework to efficiently derive the optimal policy from diverse, suboptimal demonstrations. Experiments in both online and offline IL settings, with simulated and human-generated data, demonstrate IRLEED's adaptability and effectiveness, making it a versatile solution for learning from suboptimal demonstrations.
\end{abstract}


\begin{links}
    \link{Code}{https://github.com/mbeliaev1/IRLEED}
    \link{Extended version}{https://arxiv.org/abs/2402.01886}
\end{links}



	\begin{figure*}[!t]
		\centering
        \includegraphics[width=0.8\textwidth]{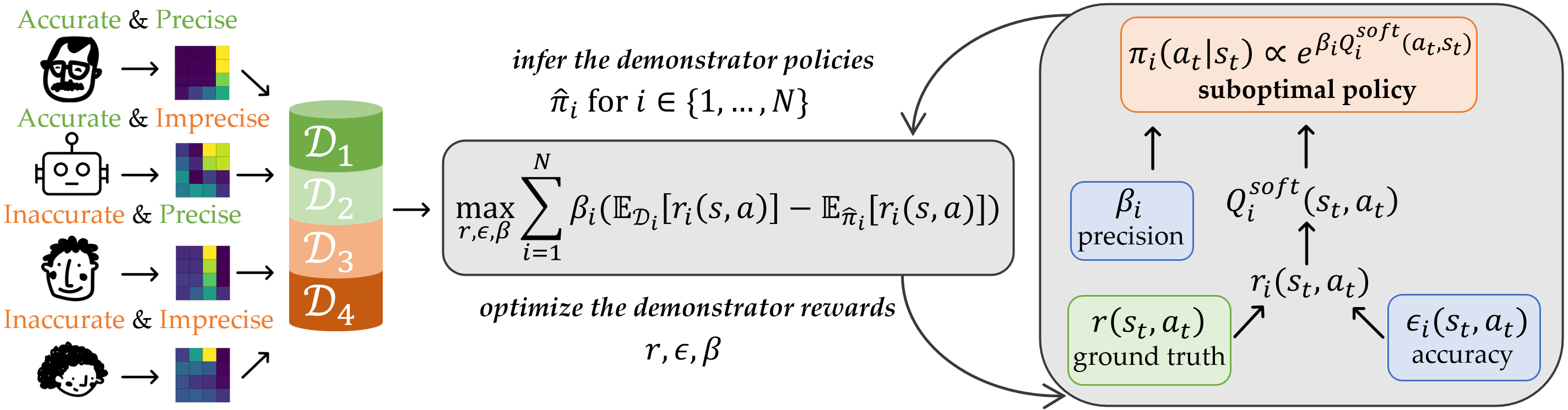}
		\caption{IRLEED is applied in the suboptimal setting to estimate the ground truth reward $r$, which is used to find the optimal policy. \textit{Left}: A heterogeneous dataset is collected from multiple sources with varying optimality. We categorize this optimality by using accuracy to represent the reward bias, and precision to represent the variance in action choices. \textit{Right}: We infer the demonstrator policies using a model for behavior based on the Boltzmann rationality principle, which captures both the accuracy $\epsilon_i$, and the precision $\beta_i$, as compared to the ground truth reward $r$. \textit{Middle}: Using estimates of the demonstrator policies $\hat{\pi}_i$ along with the demonstrations $\demons_i$, we can optimize for the true reward $r$, and the parameters that capture accuracy $\epsilon$ and precision $\beta$.} 
		\label{fig:cover}
	\end{figure*}
 
	\section{Introduction}\label{sec: Introduction}	
	
	Reinforcement Learning (RL) has proven to be a powerful tool across a wide range of applications, from controlling robotic systems, to playing complex games like Go and Chess. The efficacy of RL algorithms is largely attributed to their ability to optimize hand-crafted reward functions~\cite{openai_gym}. However, designing these functions is often challenging and impractical in complex environments~\cite{hadfield2017inverse}. A prevalent alternative to this dilemma is Imitation Learning (IL), which bypasses the need for cumbersome reward engineering by leveraging expert demonstrations to instill desired behaviors~\cite{argall2009survey}.\par
	
	The two main approaches utilized in IL are: behavioral cloning~\cite{pomerleau1991efficient}, which acquires a policy through a supervised learning approach using state-action pairs provided by the expert, and Inverse Reinforcement Learning (IRL)~\cite{ng2000algorithms}, which solves for the reward function that makes the expert's behavior optimal, subsequently facilitating the training of an IL policy. Despite the effectiveness of behavior cloning in simple environments with large amounts of data, learning a policy to fit single time step decisions leads to compounding errors due to covariance shift~\cite{ross2010efficient}. On the other hand, IRL learns a reward function that prioritizes entire trajectories over others, considering the sequential nature of the decision making problem at hand~\cite{abbeel_apprenticeship_2004}. Consequently, the success of IRL has spurred the development of IL techniques that either explicitly or implicitly incorporate environment dynamics~\cite{ho_generative_2016,fu_learning_2018,kostrikov2019imitation,garg_iq_learn_2022}.\par
	
	However, a critical assumption in these methodologies is the availability of high-quality demonstration data. In many practical situations, especially with crowd-sourced or varied data, the quality of demonstrations is inconsistent~\cite{robomimic2021,belkhale2023data}. For robotics, data curation becomes vital, since utilizing smaller real world datasets with noise or biases can lead to hazardous situations. Furthermore, assuming uniformity in demonstration quality overlooks the unique intentions of individual demonstrators, potentially leading to suboptimal learning outcomes~\cite{eysenbach2018diversity}. Therefore, it is critical to develop IL methods that can account for the heterogeneity and suboptimality of demonstration data. \par
	
    \noindent\textbf{Related Work}\quad Addressing the challenges posed by learning from suboptimal and heterogeneous demonstrations is complex. One line of research focuses on heterogeneous demonstrations without considering their quality~\cite{chen2020joint,chen2023fast}, assuming that all demonstrations are optimal. Other works consider suboptimal demonstrations, but require unrealistic prerequisites such as explicit knowledge of environment dynamics~\cite{brown2020better,chen_learning_2020}, a set of confidence scores over the demonstrations~\cite{zhang2021confidence,cao2021learning,wu2019imitation,pmlr-v229-kuhar23a}, or direct access to a subset of expert demonstrations~\cite{shiarlis2016inverse,xu2022discriminator,yang2021trail}.\par
    
    A recent approach that avoids these assumptions is ILEED~\cite{beliaev_imitation_2022}, which directly models the expertise of demonstrators within a supervised learning framework, thus enabling one to learn from the optimal behavior while filtering out the suboptimal behavior of each demonstrator. While this method was proven effective, it is based on a behavioral cloning formulation, which overlooks the underlying environment dynamics, and its representation of suboptimal behavior is limited, presuming demonstrators to be noisy approximations of experts.\par 
	
    \noindent\textbf{Overview}\quad To overcome these challenges, this paper introduces IRLEED, \textit{Inverse Reinforcement Learning by Estimating Expertise of Demonstrators}. As illustrated in \cref{fig:cover}, IRLEED is a novel framework for IRL that accounts for demonstrator suboptimality without prior expertise knowledge. It comprises two core components: (1) a general model of demonstrator suboptimality based on the Boltzmann rationality principle, which captures both the reward bias and the variance in action choices of each demonstrator as compared to the optimal policy (2) a Maximum Entropy IRL framework, which is used to efficiently find the optimal policy when provided with a set of suboptimal and heterogeneous demonstrations. These elements enable IRLEED to effectively recover the ground truth policy from suboptimal and heterogeneous demonstrations, surpassing the limitations of previous models. Furthermore, IRLEED's simplicity and compatibility with existing IRL techniques make it a versatile and powerful tool in the field of IL.\par

	The main contributions of our work are:
	\begin{itemize}
	\itemsep0em
        \item We propose a novel framework designed to enhance existing IRL algorithms, addressing the challenges of learning from suboptimal and heterogeneous data.
        \item We provide comparative insights against standard IRL and the behavior cloning approach, ILEED, showing how IRLEED generalizes both methods.
        \item We empirically validate the success of our method against relevant baselines in both online and offline IL scenarios.
	\end{itemize}
	
	\section{Preliminary}\label{Sec: Preliminary}
	\noindent\textbf{Notations}\quad 
	We use the Markov decision process (MDP) setting, defined by a tuple $\MDP=(\States,\Actions,\Transitions,p_0,r,\gamma)$, where $\States$, $\Actions$ represent state and action spaces, $\Transitions:\States\times\Actions\times\States\rightarrow[0,1]$ represents the dynamics, $p_0:\States\rightarrow[0,1]$ represents the initial state distribution, $r\in\Rewards:\States\times\Actions\rightarrow\reals$ represents the reward function, and $\gamma\in(0,1)$ represents the discount factor. In Sections  ~\ref{sec: Model} and~\ref{sec: IRLEED}, we describe our method using finite state and action spaces, $\States$ and $\Actions$, but our experiments later utilize continuous environments. We use $\pi:\States\times\Actions\rightarrow[0,1]$ to denote a policy that assigns probabilities to actions in $\Actions$ given states in $\States$. Considering the $\gamma$--discounted infinite time horizon setting, we define the expected value under a policy $\pi$ in terms of the trajectory it produces. Specifically, this expectation is represented as $\fev_\pi=\EX_\pi[\features(s,a)]$, which equals $\EX[\sum_{t=0}^{\infty}\gamma^t f(s_t,a_t)]$, where $s_0\sim p_0,a_t\sim\pi(\cdot|s_t)$, and $s_{t+1}\sim\Transitions(\cdot|s_t,a_t)$. Additionally, we will use $\efev_\demons=\EX_\demons[\features(s,a)]$ to denote the empirical expectation with respect to the trajectories $\traj_j=(s_0,a_0,\ldots,s_{T},a_{T})$ in demonstration set $\demons=\{\traj_j\}_{j=1}^{M}$, where we leave out denoting trajectories of varying lengths for simplicity.\par
    \noindent\textbf{Maximum Entropy IRL}\quad
    The goal of IRL is to recover a reward $r\in\Rewards$ that rationalizes the demonstrated behavior in dataset $\dataset$, where the provided dataset consists of trajectories $\traj$ sampled from an expert policy $\pi_E$. This idea was first formulated as a feature-expectation matching problem~\cite{abbeel_apprenticeship_2004}: assuming that the feature vector $\features:\States\times\Actions\rightarrow\reals^k$ components quantify the expert's behavior, to find a policy $\pi$ that performs equal to or better than the expert $\pi_E$, it suffices that their feature expectations match $\fev_\pi=\fev_{\pi_E}$. Unfortunately, this leads to an ill-posed problem as many policies can lead to the same feature counts. Maximum Entropy IRL resolves this ambiguity by choosing the policy which does not exhibit any additional preferences beyond matching feature expectations~\cite{ziebart2008maximum}:
	\begin{equation}\label{eq: IRL_objective}
		\argmax_{\pi} H(\pi),\textit{ such that: }\fev_\pi=\efev_{\demons},
	\end{equation}
	where in the infinite-time horizon setting, this equates to maximizing the discounted causal entropy $H(\pi)$ under the feature matching constraint~\cite{bloem2014infinite}.\par 

	In the MDP setting, the above problem is equivalent to finding the maximum likelihood estimate of $\theta$:
	\begin{equation}\label{eq: IRL_ML}
		\hat{\theta} = \argmax_{\theta} \log\prod_{\tau_j\in\demons}p_0(s_0)\prod_{t=0}^{T}\pi_\theta(a_t|s_t)\Transitions(s_{t+1}|s_{t},a_{t}),
	\end{equation}
	with the parameterized policy $\pi_\theta$ defined recursively using $\qsoft$ and $\vsoft$:
	\begin{align}
		\pi_{\theta}(a_t|s_t)&=\exp(\qsoft_\theta(s_t,a_t)-\vsoft_\theta(s_t)),\label{eq: IRL_pi}\\
		\qsoft_\theta&=\theta^\top\features(s_t,a_t)+\gamma\EX_{s_{t+1}\sim\Transitions(\cdot|s_t,a_t)}[\vsoft_\theta(s_{t+1})],\label{eq: q_soft}\\
		\vsoft_\theta&=\log\sum_{a_t\in\Actions}\exp(\qsoft_\theta(s_t,a_t)).\label{eq: v_soft}
	\end{align}
 
	The parameters $\theta\in\reals^k$ specifying the soft Bellman policy, $\pi_{\theta}(a_t|s_t)$, correspond to the dual variables of the feature matching constraint in Eq.~\eqref{eq: IRL_objective}. Without loss of generality, if we define the true reward signal as $r(s,a)=\theta^{\star\top}\features(s,a)$ for some reward parameter ${\theta}^\star\in\reals^k$ and features $\features$, then the soft Bellman policy parameterized by $\theta^\star$ achieves the maximum possible return (for more details see~\cite{ziebart2010thesis}, Theorem 6.2 and Corollary 6.11).\footnote{To allow an arbitrary reward specification, the IRL objective defined in Eq.~\ref{eq: IRL_objective} can be expressed in terms of occupancy measures instead of feature expectations~\cite{ho_generative_2016}.} This formulates an iterative approach to finding the true reward parameter $\theta^\star$: given demonstration set $\demons$ and your estimate $\theta$, alternate between \textit{inferring} the policy $\pi_{\theta}(a_t|s_t)$ based on Eq.~\eqref{eq: IRL_pi}, and \textit{optimizing} the parameter $\theta$ based on Eq.~\eqref{eq: IRL_ML}.
		
	This line of work provides many insights that have been used to derive modern IRL~\cite{fu_learning_2018,reddy_sqil_2019,gao2018reinforcement,garg_iq_learn_2022} and RL~\cite{haarnoja_reinforcement_2017,haarnoja_soft_2018} algorithms. Unlike these prior works, we consider the setting where we no longer have access to the expert policy $\pi_E$, and instead are given a heterogeneous set of potentially suboptimal demonstrations from which we want to learn a well performing policy. To this end, we first describe the model we use to capture the suboptimality of demonstrators in \cref{sec: Model}, following which we frame this problem within the IRL framework and derive our own method IRLEED in \cref{sec: IRLEED}. Finally, we describe our experimental results in \cref{sec: Experiments}, and conclude our work in \cref{sec: Discussion}.
	
	\section{Suboptimal Demonstrator Model} \label{sec: Model}
	This work considers the case where dataset $\demons = \{ (i, \demons_i) \}_{i=1}^{N}$ consists of a mixture of $N$ demonstrations with varying quality. Each demonstration $\demons_i=\{\traj_j\}_{j=1}^{M_i}$ contains a set of $M_i$ trajectories sampled from a distinct,  fixed policy $\pi_i$. Rather than considering $\demons$ as a homogeneous set derived from one expert policy $\pi_E$, our approach focuses on leveraging the specific source of each demonstration to enhance the IRL framework. By identifying the varying qualities within this mixed collection of demonstrations, we aim to more accurately deduce the ground truth policy $\pi_E$, as opposed to a simplistic method that averages behaviors. To achieve this, it is essential to develop a model for the demonstrators $\pi_i$ that accurately reflects their suboptimal behavior relative to the optimal policy $\pi_E$.\par
	A popular framework for modeling decision making is the Boltzmann rationality model, which is widely used in psychology~\cite{baker2009action}, economics~\cite{luce1959individual} and RL~\cite{laidlaw2022boltzmann}. In this model, the likelihood of an agent choosing an action is proportional to the exponential of the reward associated with that action: $\pi(a|s;r,\beta)\propto\exp(\beta r(s,a))$, where the inverse temperature parameter $\beta$ controls the randomness of the choice. While this formulation provides an intuitive way to capture one form of suboptimality through $\beta$, it traditionally assumes that agents have unbiased knowledge about the true reward $r(s,a)$. To broaden this model for suboptimal demonstrations, we propose considering agents that follow a biased reward $r(s,a)+\epsilon(s,a)$, introducing $\epsilon$ as the deviation from the true reward. This modification enables us to represent an agent's \textit{accuracy} through the deviation $\epsilon$, adding a bias to their perceived rewards, and their \textit{precision} through parameter $\beta$, adding variance to their action choices.\par
	
	Within the context of the IRL framework described in \cref{Sec: Preliminary}, we recall that the expert policy $\pi_E$ aligns with the true reward, formulated as $r(s,a)=\theta^{\star\top}\features(s,a)$. Applying the aforementioned concept, we can model demonstrator $i$'s perceived reward as a deviation from the true reward: $r_i(s,a)=(\theta^\star+\epsilon_i)^\top\features(s,a)$, where $\epsilon_i\in\reals^k$. By incorporating this altered reward $r_i$ into the \textit{soft} Bellman policy defined in Eq.~\eqref{eq: IRL_pi}, we derive a recursive parameterization for demonstrator $i$'s policy:
	\begin{equation}\label{eq: stochastic policy}
        \pi_{\theta^\star,\epsilon_i,\beta_i}(a_t|s_t)\propto\exp(\beta_i(\qsoft_{\theta^\star+\epsilon_i}(s_t,a_t))),
	\end{equation}
	where $\beta_i\in[0,\infty)$, and the normalizing factor not shown above should be modified to account for $\beta_i$: $\vsoft_{\theta^\star+\epsilon_i}(s_t)=\frac{1}{\beta_i}\log\sum_{a_t\in\Actions}\exp(\beta_i\qsoft_{\theta^\star+\epsilon_i}(s_t,a_t))$.\par
	
	Our model captures two distinct aspects of suboptimal behavior: (1) $\epsilon_i$ quantifies the demonstrator's \textit{accuracy} in estimating the true reward $r$, where $\epsilon_i^\top\features(\tau)$ represents the estimation error. (2) $\beta_i$ quantifies the demonstrator's \textit{precision} in action selection, where $\beta_i\rightarrow\infty$ interpolates from the soft Bellman policy, which samples actions according to their Q--values, to the standard Bellman policy, which chooses actions that maximize the Q--value. This model offers a versatile and mathematically convenient way to depict suboptimal demonstrator behaviors within the soft Bellman policy domain, as well as extend the results of Maximum Causal Entropy IRL to the suboptimal setting.\par
    \begin{remark}\label{rem: general model}
        \textit{It is important to note that this model extends outside the scope of problems where the true reward $r$ is an affine transformation of some feature vector $\features$. In the general setting, we can parameterize demonstrator $i$'s reward as $r_i=r+\epsilon_i$, a combination of the true reward $r\in\mathcal{R}$ and their deviation $\epsilon_i\in\mathcal{R}$, where $r$ and $\epsilon_i$ can be represented by neural networks to allow for general function approximation. We opt to formulate our proposed model within the classic feature matching scope, as this facilitates direct comparison with the seminal works used to derive our results in the following section. However, our proposed model's applicability extends beyond this traditional framework, as we elaborate in \cref{sec: Practical Algorithm}, and demonstrate through experiments in \cref{sec: Experiments}. Until then, we utilize the feature matching framework described, referencing it as standard IRL.}\par
    \end{remark}

	\section{Learning from Suboptimal Data with IRL}\label{sec: IRLEED}
	Up to this point, we have formulated a model for each demonstrator $i$ based on the parameter $\theta^\star$, which defines the true reward, and parameters $\epsilon_i$ and $\beta_i$, which characterize demonstrator suboptimalities. However, we do not have access to these parameters, and are instead provided with a dataset of demonstrations $\demons$. Recall that standard IRL assumes that all demonstrations come from the optimal policy $\pi_E$, and follows the iterative approach described in \cref{Sec: Preliminary} to find it. However, applying this approach to the heterogeneous regime will yield an averaged policy, leading to subpar performance if a significant portion of the demonstrations are suboptimal. To refine this, we introduce IRLEED, an innovative extension of the IRL framework tailored for the suboptimal setting. 
 
	\subsection{IRLEED}\label{sec: General Framework}
    Utilizing the dataset $\demons=\{(i,\demons_i)\}_{i=1}^{N}$, we can jointly estimate the maximum likelihood parameters:
	\begin{equation}\label{eq: IRLEED Loss}
		\loss(\theta,\epsilon,\beta) =  \log\prod_{\demons_i\in\demons}\prod_{\tau_j\in\demons_i}\prob(\traj_j|\theta,\epsilon_i,\beta_i),
	\end{equation}
	where $\epsilon=\{\epsilon_i\}_{i=1}^{N}$, $\beta=\{\beta_i\}_{i=1}^{N}$, and the probability $\prob(\traj_j|\theta,\epsilon_i,\beta_i)$ is given by: $\prob(\traj|\theta,\epsilon_i,\beta_i)=p_0(s_0)\prod_{t=0}^{T}\pi_{\theta,\epsilon_i,\beta_i}(a_t|s_t)\Transitions(s_{t+1}|s_{t},a_{t})$.
	Under the specified demonstrator model, this likelihood serves as a suitable loss function. Specifically, rewriting the joint likelihood equivalently only over $\theta$: $\loss(\theta) = \max_{\epsilon,\beta}\loss(\theta,\epsilon,\beta)$,
	it is easy to show that $\loss(\theta)$ is a proper objective function. 
	\begin{proposition}\label{prop: loss minimizer}
		$\theta^\star$ is the (non--unique) maximizer of $\loss(\theta)$.
	\end{proposition}
	
	This provides us with a convenient way to derive the gradient of likelihood function $\loss(\theta,\epsilon,\beta)$.
	
	\begin{lemma}\label{cor: main result}
		Given the likelihood $\loss$ defined according to Eq.~\eqref{eq: IRLEED Loss} and demonstrator policies $\pi_{\theta,\epsilon_i,\beta_i}$ parameterized by Eq.~\eqref{eq: stochastic policy}, we can compute the gradients of $\loss$ as follows:
		\begin{align}
			\nabla_\theta\loss &= \sum_{i=1}^{N}\beta_i(\efev_{\demons_i}-\fev_{\pi_{\theta,\epsilon_i,\beta_i}}),\label{eq: gradient_theta}\\
			\nabla_{\epsilon_i}\loss &= \beta_i(\efev_{\demons_i}-\fev_{\pi_{\theta,\epsilon_i,\beta_i}}),\label{eq: gradient_epsilon}\\
			\frac{\partial\loss}{\partial\beta_i} &= (\theta+\epsilon_i)^\top(\efev_{\demons_i}-\fev_{\pi_{\theta,\epsilon_i,\beta_i}}).\label{eq: gradient_beta}
		\end{align}
	\end{lemma}
	\begin{proof}
		First, the parameter $\beta_i$ can be folded into the modified \textit{soft} Bellman policy of Eq.~\eqref{eq: stochastic policy} as the reward estimate: $\beta_i(\theta+\epsilon_i)^\top\features (s_t,a_t)$. Second, we can express the likelihood as $\loss=\sum_{i=1}^{N}\loss_i$, where  $\loss_i=  \log\prod_{\tau_j\in\demons_i}\prob(\traj_j|\theta,\epsilon_i,\beta_i)$. Setting $u_i=\beta_i(\theta+\epsilon_i)$, we know that $\nabla_{u_i}\loss_i=\efev_{\demons_i}-\fev_{\pi_{\theta,\epsilon_i,\beta_i}}$ $\forall i\in\{1,\ldots,N\}$ (see~\cite{ziebart2010thesis}, Theorem 6.2, Lemma A.2). Using the chain rule completes the proof.
	\end{proof}
	
	This result has an intuitive interpretation. Recall that for each demonstrator $i$, $\efev_{\demons_i}-\fev_{\pi_{\theta,\epsilon_i,\beta_i}}$ represents the difference between the empirical feature vector given samples from their policy, $\efev_{\demons_i}$, and the expected feature vector under the probabilistic model, $\fev_{\pi_{\theta,\epsilon_i,\beta_i}}$. With this in mind, Eq.~\eqref{eq: gradient_epsilon} shows that we update $\epsilon_i$ to match feature expectation with demonstrator $\pi_i$, just like standard IRL. On the other hand, Eq.~\eqref{eq: gradient_theta} shows that we update $\theta$ to match a weighted average over feature expectations provided by all demonstrators, where their respective precision $\beta_i$ determines their contribution. Lastly, Eq.~\eqref{eq: gradient_beta} shows that we update this precision $\beta_i$ to balance the expected returns of the demonstrator, $\pi_i$, and our probabilistic model, $\pi_{\theta,\epsilon_i,\beta_i}$, under our estimate of the perceived reward, $r_i(s,a)=(\theta+\epsilon_i)^\top\features(s,a)$. This way, when the probabilistic model outperforms the demonstration set under the estimated reward $r_i$, we decrease the precision $\beta_i$ to lower its relative performance, and vice versa.\par
	 
	This gives us a general algorithm that extends the IRL framework. Specifically, IRLEED follows an iterative approach to finding the true reward parameter $\theta^\star$: given demonstration set $\demons=\{(i,\demons_i)\}_{i=1}^{N}$ and estimates $\theta,\epsilon,\beta$, for each demonstrator $i\in\{1,\ldots,N\}$, alternate between \textit{inferring} the policy $\pi_{\theta,\epsilon_i,\beta_i}(a_t|s_t)$ based on Eq.~\eqref{eq: stochastic policy}, and \textit{optimizing} the parameters $\theta,\epsilon_i,\beta_i$ based on Eqs.~\eqref{eq: gradient_theta},~\eqref{eq: gradient_epsilon},~\eqref{eq: gradient_beta}, respectively. 

	As aforementioned, many techniques can be used for both the inference and optimization procedures. To infer the policy $\pi_{\theta,\epsilon_i,\beta_i}(a_t|s_t)$ given $\theta,\epsilon_i,\beta_i$, one can use soft versions of the standard value iteration and Q-learning algorithms, depending on whether the dynamics $\Transitions$ are provided or not~\cite{bloem2014infinite}. To compute feature expectations $\fev_{\pi_{\theta,\epsilon_i,\beta_i}}$ based on the inferred policy, one utilize the standard dynamic programming operator when the dynamics $\Transitions$ are known, or estimate the expectations by utilizing Monte Carlo simulations of the MDP when the dynamics $\Transitions$ are unknown.\par
	
	\subsection{Comparisons to Standard IRL and ILEED}\label{sec: Comparisons}
	This likelihood maximization problem is closely related to the maximum causal entropy IRL framework described in Eqs.~\eqref{eq: IRL_objective}-~\eqref{eq: v_soft}. Specifically, IRLEED makes two modification: (1) For each $i\in\{1,\ldots,N\}$, it defines individual dual variables $\theta_i=\theta+\epsilon_i$ that correspond to feature matching constraint $\efev_{\demons_i}=\fev_{\pi_{\theta,\epsilon_i,\beta_i}}$, tying the demonstrator policies together. (2) It uses a learnable parameter $\beta_i$ to tune the magnitude of uncertainty related to demonstrations provided in $\demons_i$. Moreover, IRLEED generalizes standard IRL.\par
	\begin{remark}
		\textit{IRLEED recovers the IRL framework when we set $\epsilon_i=[0]^k$ and $\beta_i=1$ as constants.}
	\end{remark}
	
	To show the advantage of IRLEED within the suboptimal demonstration setting, we note that the policy recovered by standard IRL can only perform as well as the average demonstration $\demons_i$ provided. This is a direct result of the feature matching constraint imposed by the IRL objective. 
	
	\begin{proposition}\label{prop: IRLEED vs IRL}
		Given $\demons=\{(i,\demons_i)\}_{i=1}^{N}$ produced by suboptimal policies $\pi_i$ according to Eq.~\eqref{eq: stochastic policy}, let $\theta_{IRL}$ denote the naive solution to the IRL problem defined by Eq.~\ref{eq: IRL_ML}, which treats demonstration set $\demons$ as one homogeneous set produced by $\pi_E$. The performance of IRL is bounded by:    
		\begin{equation}\label{eq: IRLEED vs IRL}
			\EX_{\pi_{\theta_{IRL}}}[r(s,a)]=\frac{1}{N}\sum_{i=1}^{N}\EX_{\pi_i}[r(s,a)]\leq\EX_{\pi_{\theta^\star}}[r(s,a)],
		\end{equation}
		where we assume that all demonstrators provide the same number of trajectories $M_i=M$.  
	\end{proposition}	
	
	As mentioned in \cref{prop: loss minimizer}, IRLEED can recover the true reward parameter $\theta^\star$, which outperforms $\theta_{IRL}$ unless the demonstrations provided are optimal. Unfortunately, since $\theta^\star$ is not the unique solution to IRLEED, we can not make guarantees on recovering the true reward parameter $\theta^\star$. Nonetheless, our experimental results in Section~\ref{sec: Experiments} demonstrate that IRLEED improves performance over IRL frameworks in both the feature matching setting described here, as well as the general IRL setting, which can be applied to continuous control tasks with high dimensional input spaces.
		
	Finally, we note that when the MDP is deterministic, we can express the policy without recursion: $\pi_{\theta,\epsilon_i,\beta_i}(\traj)\propto\exp(\beta_i(\theta+\epsilon_i)^\top\features(\tau)),$
	where $\features(\tau)=\sum_{(s,a)\in\tau}\features(s,a)$, and $\pi(\tau)$ can directly map from trajectories due to the deterministic dynamics $\Transitions$. Replacing the probability $\prob(\traj|\theta,\epsilon_i,\beta_i)$ with $\pi_{\theta,\epsilon_i,\beta_i}(\traj)$, maximizing the likelihood defined by Eq.~\eqref{eq: IRLEED Loss} corresponds to the supervised learning approach employed by ILEED~\cite{beliaev_imitation_2022}. Hence unlike ILEED, our framework provides a dynamics--aware solution to the problem of learning from mixtures of suboptimal demonstrations. 
	
	\begin{remark}\label{rem: IRLEED vs ILEED}
		\textit{IRLEED recovers the ILEED framework if we assume all demonstrators $i$ have knowledge of the true reward, $\epsilon_i=[0]^k$, and ignore the dynamics of the MDP.} 
	\end{remark}
	
	\subsection{Practical Algorithm}\label{sec: Practical Algorithm}
	While the results in the previous sections provide insight on how to learn from suboptimal demonstrations using IRL, there are two main concerns: (1) Relying on hand-crafted features $\features$ can hinder our ability to express complex behaviors. (2) Nothing in the loss prevents the model from overfitting by learning zero shared reward $\theta$, and a completely different reward deviation $\epsilon_i$, for each demonstrator $i$. To address these concerns, we can extend our approach to generalized IRL:
 	\begin{equation}\label{eq: general_IRL}
		\max_{r\in\mathcal{R}}\min_{\pi\in\Pi}\EX_{\pi_E}[r(s,a)]-\EX_{\pi}[r(s,a)] - H(\pi) - \psi(r),
	\end{equation}
    where the reward $r$ belongs to a non restrictive set of functions $\mathcal{R}=\{r:\States\times\Actions\rightarrow\reals\}$, and $\psi$ is a convex regularizer that is used to prevent overfitting~\cite{ho_generative_2016}.\par 
    
    In brief, we extend Eq.~\eqref{eq: general_IRL} to the suboptimal setting by using neural networks to parameterize $r$ and $\epsilon=\{\epsilon_i\}_{i=1}^{N}$, and implement $\ell_2$--regularization on the outputs of $\epsilon$ to mitigate overfitting. Our experiment results in Section~\ref{sec: Experiments} demonstrate that IRLEED performs well in this generalized setting, and can even be trained offline when paired with Inverse soft-Q learning (IQ), an approach that avoids the iterative process defined by learning a single Q-function, implicitly representing both reward and policy~\cite{garg_iq_learn_2022}. We leave the details of this section to \cref{sec: App Practical Algorithm}, covering how IRLEED can be utilized alongside generalized IRL algorithms to address both of the aforementioned concerns.\par

    \begin{figure}[!t]
		\centering
        \includegraphics[width=.9\columnwidth]{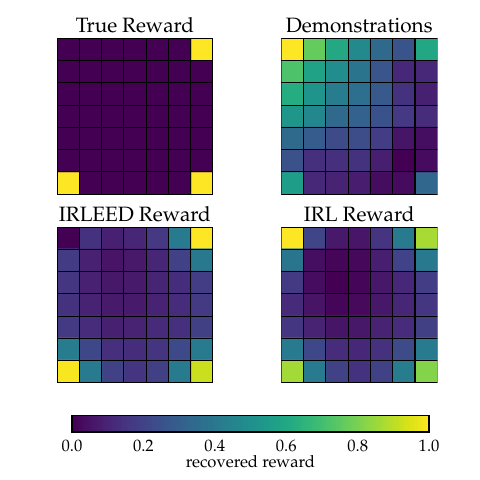}
		\caption{We visualize the reward recovered by IRLEED and IRL when trained using suboptimal demonstrations. Top left shows the true reward, where the three yellow corners are terminal states. Top right shows the normalized state visitation frequency over the entire dataset. Bottom left and right show the normalized rewards recovered by IRLEED and IRL respectively. We can see that the provided demonstrations are misaligned with the ground truth reward: the state visitation frequency for the top left corner is higher due to demonstrator suboptimalities. As expected, the feature matching constraint of IRL absorbs this suboptimality. Although the reward recovered by IRL contains information about the true reward, it is incorrectly biasing the top left corner. On the other hand, IRLEED is able to remove this bias, providing a better estimate of the ground truth reward.}
        \label{fig:IRL_vs_IRLEED_rec}
    \end{figure}
    
	\begin{figure}[!t]
 		\includegraphics[width=.9\columnwidth]{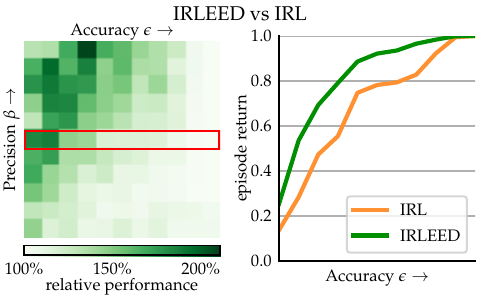}
		\caption{We compare the performance of the policies recovered by IRLEED and IRL. The left plot shows the relative performance of IRLEED over IRL under varying dataset settings, where the top right corner corresponds to expert data. The right plot shows the performance of both policies as we increase the accuracy $\epsilon$ of demonstrators in the dataset, corresponding to the data outlined in red in the left plot. On average, IRLEED provided a $30.3\%$ improvement over IRL.}
        \label{fig:IRL_vs_IRLEED}
	\end{figure}

 	\begin{table*}[!t]
	\centering
	\setlength{\tabcolsep}{2.0pt}
	\begin{tabular}{cc||c|cc|ccc||c|cc|ccc}
		\toprule
		\multicolumn{2}{c}{Setting} & \multicolumn{6}{c}{Cartpole} & \multicolumn{6}{c}{Lunar Lander}\\
		\cmidrule(r){1-2}\cmidrule(r){3-8}\cmidrule(r){9-14}
		Prec. & Acc. & Mean & \textit{IQ}$^*$ & \textit{IRLEED}$^*$ & ILEED & IQ & IRLEED & Mean & \textit{IQ}$^*$ & \textit{IRLEED}$^*$ & ILEED & IQ & IRLEED\\
		\textit{L} & \textit{L} & $0.58$ & $0.97$ & $\textbf{1.55}$ & $1.38$ & $0.96$ & $\textbf{1.56}$ & $0.46$ & $0.68$ & $\textbf{1.28}$ & $0.84$ & $0.81$ & $\textbf{1.34}$ \\
		\textit{L} & \textit{H} & $0.70$ & $0.92$ & $\textbf{1.00}$ & $1.00$ & $0.96$ & $\textbf{1.00}$ & $0.54$ & $0.84$ & $\textbf{1.14}$ & $1.06$ & $1.05$ & $\textbf{1.18}$ \\
		\textit{H} & \textit{L} & $0.64$ & $1.03$ & $\textbf{1.09}$ & $1.05$ & $0.95$ & $\textbf{1.05}$ & $0.57$ & $0.94$ & $\textbf{1.30}$ & $0.85$ & $0.99$ & $\textbf{1.14}$ \\
		\textit{H} & \textit{H} & $1.00$ & $0.99$ & $1.00$ & $0.92$ & $1.00$ & $1.00$ & $0.86$ & $0.93$ & $\textbf{0.98}$ & $1.01$ & $0.99$ & $0.98$ \\
		\bottomrule
	\end{tabular}
	\caption{Simulated Control Tasks}
	\raggedright We list the average return of the recovered policies (IQ, IRLEED, ILEED), relative to the best demonstrator's return, where \textit{IQ}$^*$ and \textit{IRLEED}$^*$ are used to denote \textit{online} algorithms. This is provided for $4$ dataset settings (Prec., Acc.), where we list the mean performance (Mean) of all demonstrators relative to the best. The last row (H,H) corresponds to the clean dataset. The return of the best demonstrator for each setting (top to bottom) is $260,500,411,500$ (Cartpole) and $171,213,166,277$ (Lunar Lander).
        \label{tab:cartpole}
	\end{table*}

	\section{Experiments} \label{sec: Experiments}
    In this section we evaluate how IRLEED performs when learning from suboptimal demonstrations, using experiments in both online and offline IL settings, with simulated and human-generated data. Throughout this section we compared IRLEED to maximum entropy IRL~\cite{ziebart2010modeling}, Inverse soft-Q learning (IQ)~\cite{garg_iq_learn_2022}, and ILEED~\cite{beliaev_imitation_2022}. In total, we performed four sets of experiments: (1) Using a custom Gridworld environment, we simulated suboptimal demonstrations with varying levels of \textit{precision} and \textit{accuracy}, and compared the performance and recovered rewards of IRLEED and maximum entropy IRL. (2) Using the control tasks, Cartpole, and Lunar Lander, we simulated suboptimal demonstrations with varying levels of \textit{expertise} and \textit{accuracy}, and compared the performance of IRLEED, IQ, and ILEED in both the online and offline settings. (3) Using the Atari environments, Space Invaders, and Qbert, we combined both simulated and human demonstrations, and compared the performance of IRLEED, IQ, and ILEED in the offline setting. (4) Using the continuous control Mujoco task Hopper, we collected suboptimal demonstrations from pretrained policies, and compared the performance of IRLEED and IQ.\par

    \noindent\textbf{Implementation}\quad 
    We utilized the codebases provided by the authors to implement ILEED and IQ, creating IRLEED ontop of the IQ algorithm by adding a learnable parameter $\beta_i$, and an additional critic network $\epsilon_i$, for each demonstrator $i$. Since IQ learns a reward and policy directly by using a single Q function as a critic, we used our additional critic network $\epsilon_i$ to directly add bias to the state action function instead of the reward. For further implementation details refer to \cref{Sec: Appendix Implementation Details}.

    \noindent\textbf{1. Gridworld}\quad
    For the Gridworld experiments, we implemented IRL and IRLEED by utilizing stochastic value iteration on the reward estimate $\theta$ to infer the policy $\pi$, and Monte Carlo simulations of the MDP to estimate the feature expectations required to optimize $\theta$, repeating both steps until convergence. For IRLEED, we updated $\epsilon$ and $\beta$ only after $\theta$ converged, and repeated this for two iterations.
    
    The reward in this setting was linearly parameterized as $r(s)=\theta^\top\features(s)$, where $\features(s)$ is a one hot encoding for each gridworld state. For each level of \textit{precision} and \textit{accuracy}, we generated $5$ suboptimal policies according to \cref{eq: stochastic policy}, where $\beta_i$ was sampled from a uniform distribution starting at $0$, with mean equal to the \textit{precision} level, and $\epsilon_i$ was sampled from a multivariate normal distribution $\mathcal{N}([0]^k,\mathbf{I}_k/\lambda^2)$, with $\lambda$ equal to the \textit{accuracy} level. We tested $121$ dataset settings, collecting $40$ trajectories from each policy, and using $100$ seeds for each setting. Note that we visualized the recovered reward $r(s,a)$ for both IRLEED and IRL under one dataset setting in \cref{fig:IRL_vs_IRLEED_rec}. Given demonstrations that are misaligned with the ground truth reward, we saw that the feature matching constraint of IRL absorbed this suboptimality, while IRLEED was able to remove this bias, providing us with a better estimate of the ground truth reward.\par
     
    To see the effect of dataset quality on performance, we use \cref{fig:IRL_vs_IRLEED} to show the relative improvement of the policy recovered by IRLEED over maximum entropy IRL for each dataset setting. We can see from the left plot that IRLEED provides improvement over IRL as we decrease \textit{accuracy} $\epsilon$ (by decreasing $\lambda$), especially when \textit{precision} $\beta$ is high. This signifies the importance of modeling the accuracy $\epsilon$, or reward bias, in addition to the precision $\beta$, or action variance, when learning from suboptimal demonstrations. This point is further highlighted in the right plot, which shows that while both methods can achieve good performance under low precision settings, IRLEED can tolerate lower demonstrator accuracy compared to IRL.\par
            
    \noindent\textbf{2. Simulated Control Tasks}\quad
    For this experiment, each level of \textit{precision} and \textit{accuracy} corresponded to $3$ suboptimal policies generated according to $\pi_i(a|s)\propto\exp(\beta_i(Q(s,a)+\alpha\epsilon_i(s,a)))$, where $Q(s,a)$ is a pretrained critic network, $\epsilon_i(s,a)$ is a randomly initialized critic network, and $\alpha$ is a scaling factor which is set to zero for the high \textit{accuracy} setting. For the low \textit{precision} setting, $\beta_i$ was sampled from a uniform distribution, whereas for the high \textit{precision} setting, the stochasticity of the policy was removed by following the maximum state action value. We utilized both the online and offline setting, using $30$ seeds for each dataset setting, and displaying our results in \cref{tab:cartpole}. Note that with near binary performance of the recovered policies, the variance is due to demonstration quality, hence we report averages relative to the best demonstrator computed over all initializations.\par   

    We can see that in the online setting, IRLEED outperforms its ablated counterpart, IQ, under all $3$ suboptimal dataset settings (not including the clean dataset), for both environments. Furthermore, we see that in the more complex Lunar Lander environment, IQ can not perform as well as the best demonstrator, whereas IRLEED consistently performs on par or better than the best demonstrator. To have a fair comparison between IRLEED and its behavior cloning counterpart, ILEED, we repeat these experiments without online access to environment interactions. We can see that offline IRLEED outperforms ILEED, under all $3$ suboptimal dataset settings, for both environments. Furthermore, although ILEED has comparable performance in the simpler Cartpole environment, it performs poorly in the more complex Lunar Lander environment. Moreover, we can see that ILEED shows diminished performance when demonstrators have low accuracy $\epsilon$, signifying that it can not account for the reward biases present in these settings, as was highlighted earlier in \cref{rem: IRLEED vs ILEED}. \par 
        
    \noindent\textbf{3. Atari with Human Data}\quad
    Our demonstrations comprised two unique sources, dataset \textit{A}: simulated data using a pretrained expert policy and dataset \textit{B}: collected data using adept human players~\cite{kurin2017atari}. Since the human dataset \textit{B} was acquired using screenshots of web-browsers, there is a mismatch in sampling frequencies and colors compared to the simulated environment used in dataset \textit{A}. This provides a challenge akin to crowd-sourcing data from varying sources, where we test if IRLEED can improve IQ when training on the combined dataset \textit{AB} as opposed to dataset \textit{A} alone.\par
    
    The results shown in \cref{tab:space_invaders} list the return of the recovered policies for both environments (Space Invaders and Qbert), averaged over $5$ seeds. While using dataset \textit{A} is enough to reach decent performance, all methods do poorly when trained exclusively on dataset \textit{B} due to the mismatch between the simulation environment and the data. However, we can see that IRLEED significantly improves its performance when utilizing the combined dataset \textit{AB}. Overall we see IRLEED shows a $17\%$ improvement over IQ, and a $41\%$ improvement over ILEED in the combined \textit{AB} setting.\par

    \begin{table}[!b]
		\centering
		\setlength{\tabcolsep}{2.5pt}
		\begin{tabular}{c||ccc|ccc}
			\toprule
			\multicolumn{1}{c}{Dataset} & \multicolumn{3}{c}{Space Invaders} & \multicolumn{3}{c}{Qbert}\\
			\cmidrule(r){1-1}\cmidrule(r){2-4}\cmidrule(r){5-7}
			& ILEED & IQ & IRLEED & ILEED & IQ & IRLEED\\
			\textit{AB} & $802$ & $755$ & $\textbf{910}$ & $5511$ & $8182$ & $\textbf{9336}$\\
			\textit{A}  & $783$ & $754$ & $768$ & $5926$ & $8092$ & $8049$\\
			\textit{B}  & $110$ & $179$ & $123$ & $0$ & $809$ & $1087$\\
			\bottomrule
		\end{tabular}
		\caption{Utilizing Demonstrations from Varying Sources}
		\raggedright We list the average return of the recovered policies, using the simulated dataset \textit{A}, human dataset \textit{B}~\cite{kurin2017atari}, and their combination \textit{AB}. We note that IRLEED maintains similar per-seed variance as IQ on \textit{AB} - Space Invaders: $61$ vs $54$ (IQ), Qbert: $315$ vs $376$ (IQ).  
        \label{tab:space_invaders}
	\end{table}

    \begin{figure}[!t]
		\centering
        \includegraphics[width=.9\columnwidth]{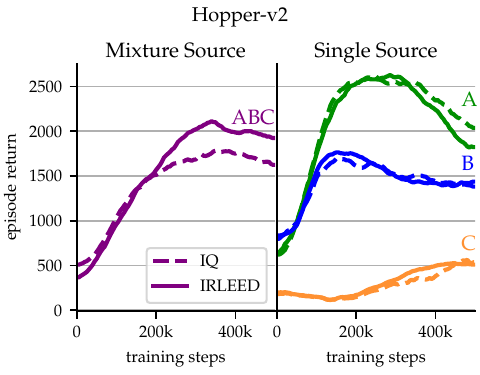}
		\caption{We plot the mean episode return of the policies learned by IRLEED (solid) and IQ (dashed) as they train on data from a single source (right) vs a mixture source (left).}
        \label{fig:hopper}
    \end{figure}
    \noindent\textbf{4. Mujoco}\quad
    For our results on the Mujoco task Hopper, we collect a single trajectory from three sources, dataset \textit{A}: a pretrained expert policy (return of $3457$), dataset \textit{B}: a partially trained policy (return of $2121$) and a barely trained policy \textit{C} (return of $216$). Due to the complex nature of this continuous control tasks, the partially trained policies can not provide adequate demonstrations for learning. Hence our goal was to see if IRLEED could improve the capability of IQ when training on the combined dataset \textit{ABC} by filtering the suboptimal demonstrations.\par
    
    The results shown in \cref{fig:hopper} plot the mean episode return of the policies during training, averaged over $15$ initializations. As expected, both methods do poorly when training on dataset \textit{C}, while achieving improved performance when training on dataset \textit{B}. However, we see IRLEED outperforms IQ when training on the combined dataset \textit{ABC} since it is more robust to the suboptimal demonstrations present in the data. However, we note that there is still a gap between the desired performance we see when training on dataset \textit{A} alone, which demonstrates the difficulty of the task at hand.

	\section{Conclusion} \label{sec: Discussion}
    \noindent\textbf{Summary}\quad This paper addresses the challenges of leveraging suboptimal and heterogeneous demonstrations in IL by introducing IRLEED, a novel IRL framework. By integrating a general model of demonstrator suboptimality that accounts for reward bias and action variance, with a maximum entropy IRL framework, IRLEED mitigates the limitations inherent in standard IL algorithms, which often fail to account for the diversity and imperfections in real-world data. Our key insight is that: \textit{supervised learning frameworks should account for the heterogeneity in crowd-sourced datasets by leveraging information about the source identity during training.}\par
    \noindent\textbf{Limitations}\quad Although we have shown that IRLEED is a general framework that improves over IRL in a multitude of settings, design choices such as the demonstrator model and reward regularization can be further analyzed. While we do try to make our framework easy to interpret, there is inherent complexity added as we need to simultaneously learn multiple demonstrator models. We also note that apart from the Gridworld experiments, we did not analyze reward recovery as it is hard to interpret in complex environments. However, our work addresses a problem setting where achieving desirable performance is a challenge in itself, and many recovered policies are below optimal. Lastly, we have not provided guarantees that state whether IRLEED will recover a better reward estimate compared to IRL. While we believe this is a good direction for future work, our goal was to present a novel formulation easily integrated with existing IRL methods, where further theoretical analysis is only possible under additional assumptions, and not crucial given the nonconvexity of the general setting.\par 
    
\section*{Acknowledgements}
This research was supported by NSF ECCS grant \#2419982.

\bibliography{refs.bib}  

    \clearpage
    
	\appendix
	\section{Extension to Generalized IRL}\label{sec: App Practical Algorithm}
	While the results in \cref{sec: IRLEED} provide insight on how to learn from suboptimal demonstrations using IRL, relying on hand-crafted features $\features$ can hinder our ability to express complex behavior. In general, the true reward $r$ can belong to a non restrictive set of functions $\mathcal{R}=\{r:\States\times\Actions\rightarrow\reals\}$. In this case, one can utilize the generalized IRL objective instead:
	\begin{equation}\label{eq: general_IRL_2}
		\max_{r\in\mathcal{R}}\min_{\pi\in\Pi}\EX_{\pi_E}[r(s,a)]-\EX_{\pi}[r(s,a)] - H(\pi) - \psi(r),
	\end{equation}
	where $\psi$ is a convex reward regularizer that smoothly penalizes differences between occupancy measures to prevent overfitting~\cite{ho_generative_2016}. A naive solution to this nested min-max objective involves an outer loop learning the reward, and an inner loop executing maximum entropy RL with this reward to find the optimal policy. As before, the optimal policy under reward $r$ takes the form of the soft Bellman policy defined in \cref{eq: IRL_pi}, where $r(s,a)$ replaces the linear reward $\theta^\top\features(s,a)$. The description of IRLEED in the generalized setting follows directly.\par 
	
	Since the suboptimal demonstrator model defined in \cref{eq: stochastic policy} extends to this setting, we can parameterize our estimate of demonstrator $i$'s policy with the soft Bellman policy $\hat{\pi}_i$ under the reward estimate $r_i=r+\epsilon_i$ and precision $\beta_i$. Utilizing the dataset $\demons=\{(i,\demons_i)\}_{i=1}^{N}$, the outer loop of the generalized IRLEED objective is  simplified to:
	\begin{equation}\label{eq: general_IRLEED}
		\max_{r,\epsilon,\beta}\sum_{i=1}^{N}\beta_i\big(\EX_{\demons_i}[r_i(s,a)]-\EX_{\hat{\pi}_i}[r_i(s,a)]\big)-\psi(\beta_i r_i),
	\end{equation}
	where in practice, the unknown functions $r$ and $\epsilon=\{\epsilon_i\}_{i=1}^{N}$ can be parameterized by neural networks. Note that under a constant reward regularizer $\psi$, if we express the rewards as $r_i(s,a)=(\theta+\epsilon_i)^\top\features(s,a)$, then the gradients of the above objective are identical to the ones described in $\cref{cor: main result}$. We note that in addition to the reward regularization used in Eq.~\eqref{eq: general_IRLEED}, we use $\ell_2$-regularization on the output of neural networks $\{\epsilon_i\}_{i=1}^{N}$ to model the assumption that each demonstrator $i$ has a reward $r_i$ that varies slightly from the common reward $r$. One direction for future analysis would be to experiment with other regularization techniques, and see which assumptions are better for different kinds of data.\par 
 
	\section{Implementation Details}\label{Sec: Appendix Implementation Details}
    Below we detail the experimental setups utilized in \cref{sec: Experiments}, and provide information on the computational resources used for our work. Note that we provide the implementation of IRLEED used during the Gridworld experiments, while referencing implementations used for other experiments below.\par
    
    \subsection{Gridworld}\label{Sec: App Gridworld}
    To implement both IRLEED and IRL for the Gridworld experiments, we utilized stochastic value iteration on the reward estimate $\theta$ to infer the policy $\pi$, and Monte Carlo simulations of the MDP to estimate the feature expectations required to optimize $\theta$, repeating both steps until convergence. As stated, when using IRLEED, we only updated $\epsilon$ and $\beta$ after $\theta$ converged, and repeated this for two iterations. For stochastic value iteration, we used a discount factor of $0.9$, a maximum horizon of $100$ timesteps, and a convergence criteria of $1e-4$ on the inferred state action values. For Monte Carlo simulations, we used $100$ episode samples. For optimization, we used stochastic gradient ascent with learning rates of $0.05, 0.1, 0.2$ for $\beta$, $\epsilon$, $\theta$, respectively, and a convergence criteria of $1e-4$ on $\theta$. For IRL, we initialized $\theta=[0.1]^k$, and set $\beta_i=1$ and $\epsilon_i=[0]^k$ to constants for all demonstrators, removing their effect. For evaluation, we used the policy recovered from the final estimate of $\theta$, sampling $100$ episodes to measure the mean reward.\par

    To create datasets for varying levels of \textit{precision} and \textit{accuracy}, we randomly generated $5$ suboptimal policies according to \cref{eq: stochastic policy}, where $\beta_i$ was sampled from a uniform distribution starting at $0$, with mean equal to the \textit{precision} level, and $\epsilon_i$ was sampled from a multivariate normal distribution $\mathcal{N}([0]^k,\mathbf{I}_k/\lambda^2)$, with $\lambda$ equal to the \textit{accuracy} level. For $\beta$, the maximum values used for the uniform distribution were: $0.4,0.5,1,1.5,2,2.5,3,3.5,4,4.5,5$. For $\lambda$, the values used were: $2,2.5,3,3.5,4,4.5,5,5.5,6,10,\infty$, where $\infty$ corresponds to setting $\epsilon_i=[0]^k$ to remove its effect. We tested a total of $121$ dataset settings, collecting $40$ trajectories from each policy to compose our dataset, and comparing the performance of IRLEED and IRL over $100$ seeds for each setting.\par
    
    \subsection{Simulated Control Tasks}\label{Sec: App Control}
    To implement IQ and setup the training and evaluation procedures for the simulated control tasks, we utilized the codebase provided by the authors.\footnote{https://github.com/Div99/IQ-Learn} For all experiments, we utilized the provided hyperparameters for IQ. To implement ILEED, we followed the codebase provided by the authors, implementing the modified behavior cloning loss alongside the IQ implementation.\footnote{https://github.com/Stanford-ILIAD/ILEED} For IRLEED, we built on top of the aforementioned IQ implementation by adding a learnable parameter $\beta_i$, and an additional critic network $\epsilon_i$, for each demonstrator $i$. This way, we sampled suboptimal policies according to: $\pi_i(a|s)\propto\exp(\beta_i(Q(s,a)+\alpha\epsilon_i(s,a))$. All three methods utilized the same architecture for the critic networks, and followed the original hyperparamters used for the IQ implementation. For ILEED, the critic network was used to compute action probabilities using a final softmax layer. For IRLEED and ILEED, $\beta_i$ was initialized to the same inverse temperature parameter used in the IQ implementation, and the learning rate was set to $5e-4$. For IRLEED, the additional critic networks $\epsilon_i(s,a)$ were initialized with the same procedure as the critic network $Q(s,a)$, utilized the same learning rate divided by $10$, a scaling factor of $\alpha=0.01$, and an $\ell_2$-regularization weight of $1e-2$. We trained all methods using $50,000$ and $100,000$ time steps for Cartpole and Lunar Lander, respectively. For evaluation, we utilized the final policy, sampling $300$ episodes to measure the mean reward.\par  

    To create datasets for varying levels of \textit{precision} and \textit{accuracy}, we randomly generated $3$ suboptimal policies according to $\pi_i(a|s)\propto\exp(\beta_i(Q(s,a)+\alpha\epsilon_i(s,a))$, where $Q(s,a)$ was a pretrained critic network provided by the original IQ implementation, and $\epsilon_i(s,a)$ was a randomly initialized critic network. For Cartpole, we utilized $\alpha=0$ and $\alpha=20$ for the high and low \textit{accuracy} settings, respectively. For Lunar Lander, we utilized $\alpha=0$ and $\alpha=0.2$ for the high and low \textit{accuracy} settings, respectively. For the low \textit{precision} setting, $\beta_i$ was sampled from a uniform distribution between $0$ and $1$ for Cartpole, and between $0$ and $100$ for Lunar Lander. For the high \textit{precision} setting, the effect of $\beta$ was removed and the policies were sampled according to the maximum state action values: $\pi_i(a|s)=\max(Q(s,a)+\alpha\epsilon_i(s,a))$. We ran this experiment in both the online and offline setting, comparing the performance of IRLEED, IQ, and ILEED, using $30$ seeds for each dataset setting.\par   

    \subsection{Atari with Human Data}\label{Sec: App Atari}
    As aforementioned, the implementation for the Atari experiments was identical to the one described above in \cref{Sec: App Control}. For training, we utilized $1,000,000$ time steps for all methods. For IRLEED, we did not update the error critic networks $\epsilon_i(s,a)$ or the temperature parameters $\beta_i$ for the first $300,000$ timesteps. For evaluation, we sampled $10$ episodes from the policies recovered during the last $5$ epochs of training, where each epoch corresponded to $5,000$ time steps.\par
    
    The two datasets used are publicly available, where dataset \textit{A}: simulated data using a pretrained expert policy, was provided alongside the original IQ implementation~\cite{garg_iq_learn_2022} and dataset \textit{B}: collected data using adept human players, was part of a larger human study~\cite{kurin2017atari}. For Space Invaders, dataset \textit{A} contained $5$ trajectories with an average return of $1285$, whereas dataset \textit{B} contained $5$ trajectories with an average return of $1801$. For Qbert, dataset \textit{A} contained $10$ trajectories with an average return of $14760$, whereas dataset \textit{B} contained $10$ trajectories with an average return of $17340$. We ran this experiments using the three combinations of these datasets, averaging our results over $5$ seeds.\par

    \subsection{Mujoco}\label{Sec: App Mujoco}
    To train on the continuous control Mujoco environments, we used the SAC implementation of IQ provided by the authors. For IRLEED, we built on top of the aforementioned IQ implementation by adding a learnable parameter $\beta_i$, and appending an additional network $\epsilon_i$ to the actor for each demonstrator $i$, keeping the critic network identical. All three methods utilized the same architectures, and followed the original hyperparamters used for the IQ implementation. For IRLEED, $\beta_i$ was initialized to the same inverse temperature parameter used in the IQ implementation, and the learning rate was set to $1e-5$. The additional actor networks $\epsilon_i(s,a)$ were initialized with the same procedure as the actor network $\pi(s,a)$, utilized the same learning rate divided by $10$, a scaling factor of $\alpha=0.01$, and an $\ell_2$-regularization weight of $1e-2$. We trained all methods using $500,000$ time steps. For evaluation, we sampled $10$ episodes from the policies recovered during training, abd averaged our result over $15$ initializations.\par  

    To create the different datasets consisting of \textit{A}, \textit{B}, and \textit{C}, we utilized pretrained policies. Utilizing the demonstrations provided by the IQ codebase, we ran the IQ algorithm for $1,000,000$ timesteps, $50,000$ timesteps, and $1,000$ timesteps, creating three policies with varying optimality (corresponding to datasets \textit{A}, \textit{B}, and \textit{C}, respectively). When creating the combined datasets, each policy was used to generate $1000$ state-action pairs, some containing more than one trajectory.  

    \subsection{Computational Resources}\label{Sec: App Compute}
    All of the experiments were performed on an shared cluster containing the 20C/40T Intel Xeon Silver 4114 CPU, 64GB RAM, and 4×GTX-1080 GPUs.  

\end{document}